\documentclass[10pt,letterpaper,conference]{ieeeconf}
\IEEEoverridecommandlockouts
\usepackage{amsmath,amsfonts}
\usepackage{algorithm}
\usepackage{array}
\usepackage{textcomp}
\usepackage{stfloats}
\usepackage{url}
\usepackage{verbatim}
\usepackage{graphicx}
\usepackage{cite}

\usepackage{accents}
\usepackage[noend]{algpseudocode}
\usepackage{bm}
\usepackage{cite}
\usepackage{hyperref}
\usepackage{ifthen}
\usepackage[utf8]{inputenc}
\usepackage{mathtools}
\usepackage{scalerel}
\usepackage{svg}
\usepackage{subcaption}
\usepackage{tikz}

\newboolean{anonymize}
\setboolean{anonymize}{false}

\usetikzlibrary{svg.path}
\definecolor{orcidlogocol}{HTML}{A6CE39}
\tikzset{
  orcidlogo/.pic={
    \fill[orcidlogocol] svg{M256,128c0,70.7-57.3,128-128,128C57.3,256,0,198.7,0,128C0,57.3,57.3,0,128,0C198.7,0,256,57.3,256,128z};
    \fill[white] svg{M86.3,186.2H70.9V79.1h15.4v48.4V186.2z}
                 svg{M108.9,79.1h41.6c39.6,0,57,28.3,57,53.6c0,27.5-21.5,53.6-56.8,53.6h-41.8V79.1z M124.3,172.4h24.5c34.9,0,42.9-26.5,42.9-39.7c0-21.5-13.7-39.7-43.7-39.7h-23.7V172.4z}
                 svg{M88.7,56.8c0,5.5-4.5,10.1-10.1,10.1c-5.6,0-10.1-4.6-10.1-10.1c0-5.6,4.5-10.1,10.1-10.1C84.2,46.7,88.7,51.3,88.7,56.8z};
  }
}
\newcommand\orcidicon[1]{\href{https://orcid.org/#1}{\mbox{\scalerel*{
\begin{tikzpicture}[yscale=-1,transform shape]
\pic{orcidlogo};
\end{tikzpicture}
}{|}}}}

\newcommand{\norm}[1]{\lVert#1\rVert}

\algrenewcommand\alglinenumber[1]{\footnotesize\textbf{#1}}
\algblockdefx{IndentBlock}{EndIndentBlock}{\vspace{-\baselineskip}}{\vspace{-\baselineskip}}
\algblockdefx{InnerIndentBlock}{EndInnerIndentBlock}{\vspace{-\baselineskip}}{\vspace{-\baselineskip}}
\newcommand{\StateNoIndent}{\Statex\hspace{-\algorithmicindent}}

\newtheorem{theorem}{Theorem}

\newtheorem{proposition}{Proposition}
\newtheorem{remark}{Remark}
\def\p{{\bm{p}}}

\def\pg{{\bm{p}^\mathrm{g}}}
\def\x{{\bm{x}}}
\def\dx{{\dot{\bm{x}}}}

\def\hx{{\hat{\bm{x}}}}
\def\dhx{{\dot{\hat{\bm{x}}}}}

\def\bu{{\bm{u}}}

\def\y{{\bm{y}}}

\def\z{{\bm{z}}}
\def\dz{{\dot{\bm{z}}}}

\def\r{{\bm{r}}}
\def\xr{{\bm{x}^\mathrm{r}}}

\def\ur{{\bm{u}^\mathrm{r}}}
\def\bv{{\bm{v}}}

\def\ds{{\dot{s}}}

\def\w{{\bm{w}}}

\def\hw{{\hat{w}}}
\def\bwo{{\bar{w}^\mathrm{o}}}

\def\bmeta{{\bm{\eta}}}
\def\hbmeta{{\hat{\bmeta}}}
\def\hw{{\hat{\bm{w}}}}

\def\hwb{{\hat{\bm{w}}^\mathrm{b}}}

\def\hv{{\hat{\bm{v}}}}

\def\tr{^\mathrm{r}}
\def\bdelta{{\bm{\delta}}}

\def\beps{{\bm{\epsilon}}}

\def\bzeta{{\bm{\zeta}}}

\def\fw{{f^\mathrm{w}}}

\def\bzeta{{\bm{\zeta}}}

\def\Ts{{T^\mathrm{s}}}
\def\0t{{0|t}}
\def\tt{{\tau|t}}
\def\tpt{{t+\tau}}
\def\ct{{\cdot|t}}

\def\Tt{{T|t}}
\def\tpT{{t+T}}
\def\Tpt{{T+\tau}}

\def\tpTsT{{t+\Ts+T}}

\def\TsT{{\Ts+T}}

\def\TTt{{\TsT|t}}

\def\gjs{{g_{j}^\mathrm{s}}}

\def\Ljs{{L_{j}^\mathrm{s}}}

\def\ljs{{l_{j}^\mathrm{s}}}

\def\cjs{{c_{j}^\mathrm{s}}}
\def\cjcs{{c_{j}^\mathrm{c,s}}}
\def\cjso{{c_{j}^\mathrm{s,o}}}

\def\gjtto{{g_{j,\tt}^\mathrm{o}}}

\def\Ljtto{{L_{j,\tt}^\mathrm{o}}}
\def\ljtto{{l_{j,\tt}^\mathrm{o}}}
\def\co{{c^\mathrm{o}}}
\def\cco{{c^\mathrm{c,o}}}

\def\nx{{n^\mathrm{x}}}
\def\np{{n^\mathrm{p}}}
\def\nuu{{n^\mathrm{u}}}
\def\ny{{n^\mathrm{y}}}
\def\nw{{n^\mathrm{w}}}
\def\neta{{n^\mathrm{\eta}}}

\def\ns{{n^\mathrm{s}}}
\def\no{{n^\mathrm{o}}}

\def\Inx{{I^{\nx}}}

\def\Ineta{{I^{\neta}}}
\def\R{{\mathbb{R}}}
\def\Rg0{{\mathbb{R}_{> 0}}}
\def\Rgeq0{{\mathbb{R}_{\geq 0}}}

\def\Rx{{\mathbb{R}^{\nx}}}
\def\Ru{{\mathbb{R}^{\nuu}}}
\def\Ry{{\mathbb{R}^{\ny}}}
\def\Reta{{\mathbb{R}^{\neta}}}

\def\Rxu{{\mathbb{R}^{\nx+\nuu}}}

\def\Rytx{{\mathbb{R}^{\ny\times\nx}}}

\def\Rp{{\mathbb{R}^{\np}}}

\def\N{{\mathbb{N}}}

\def\Rw{{\mathbb{R}^{\nw}}}

\def\RE{{\mathbb{R}^{\nx{\times}\nw}}}
\def\RF{{\mathbb{R}^{\ny{\times}\neta}}}
\def\RL{{\mathbb{R}^{\nx\times\ny}}}
\def\Ns{{\mathbb{N}_{[1,\ns]}}}
\def\No{{\mathbb{N}_{[1,\no]}}}
\def\B{{\mathcal{B}}}
\def\I{{\mathcal{I}}}
\def\Rc{{\mathcal{R}^\mathrm{c}}}
\def\Rr{{\mathcal{R}^\mathrm{r}}}

\def\L1{{\mathcal{L}_1}}

\def\F{{\mathcal{F}}}
\def\J{{\mathcal{J}}}

\def\U{{\mathcal{U}}}

\def\W{{\mathcal{W}}}

\def\hW{{\hat{\W}}}
\def\H{{\mathcal{H}}}
\def\hH{{\hat{\H}}}
\def\X{{\mathcal{X}}}
\def\Xf{{\mathcal{X}^\mathrm{f}}}

\def\Z{{\mathcal{Z}}}

\def\O{{\mathcal{O}}}

\def\hxz{{\hx,\z}}

\def\kappafhxr{{\kappa^\mathrm{f}(\hx,\r)}}
\def\kappad{{\kappa^\delta}}

\def\kappadhxz{{\kappad(\hxz)}}

\def\Tmeas{{T^\mathrm{M}}}
\def\Tmhe{{T^\mathrm{H}}}

\def\Jf{{\J^\mathrm{f}}}
\def\Js{{\J^\mathrm{s}}}
\def\Ts{{T^\mathrm{s}}}

\def\g{{\bm{\gamma}}}

\def\gxs{{\g^\mathrm{x}(s)}}

\def\gxss{{\frac{\partial \gxs}{\partial s}}}

\def\Vd{{V^\delta}}

\def\Vdhxz{{V^\delta(\hxz)}}

\def\Vdhxz{{V^\delta(\hxz)}}

\def\Pd{{P^\delta}}
\def\Kd{{K^\delta}}
\def\Xd{{X^\delta}}
\def\Yd{{Y^\delta}}

\def\Aarg{{A(\bzeta)}}
\def\Barg{{B(\bzeta)}}

\def\Ydarg{{Y^\delta(\x)}}

\def\Kdgs{{K^\delta(\gxs)}}

\def\opdiag{\operatorname{diag}}

\def\opmin{{\operatorname{min}}}

\def\optrace{{\operatorname{trace}}}
\def\opvert{{\operatorname{vert}}}

\def\lambdadelta{{\lambda^\delta}}
\def\lambdaeps{{\lambda^\epsilon}}
\def\lambdadeltaeps{{\lambda^{\delta,\epsilon}}}

\def\dotp{{\dot{\p}}}

\def\bv{{\bm{v}}}
\def\dv{{\dot{\bv}}}
\def\bOmega{{\bm{\Omega}}}
\def\dOmega{{\dot{\bOmega}}}

\def\btau{{\bm{\tau}}}

\def\BIRc{{\prescript{\I}{\B}{\Rc}}}

\def\BIRr{{\prescript{\I}{\B}{\Rr}}}

\title{\LARGE From Data to Safe Mobile Robot Navigation:\\An Efficient and Modular Robust MPC Design Pipeline}

\ifthenelse{\boolean{anonymize}}{
  \author{
    \thanks{Code available at https://anonymous.4open.science/r/rohmpc-anonymous-0790.}
  }%
}{
  \author{Dennis Benders$^{\orcidicon{0000-0002-6648-7128}}$, Johannes K\"{o}hler$^{\orcidicon{0000-0002-5556-604X}}$, Robert Babu\v{s}ka$^{\orcidicon{0000-0001-9578-8598}}$, Javier Alonso-Mora$^{\orcidicon{0000-0003-0058-570X}}$, and Laura Ferranti$^{\orcidicon{0000-0003-3856-6221}}$%
  \thanks{This work was supported by the National Police of the Netherlands. All content represents the opinion of the author(s), which is not necessarily shared or endorsed by their respective employers and/or sponsors. Laura Ferranti received support from the Dutch Science Foundation NWO-TTW Foundation within the Veni project HARMONIA (nr. 18165).}%
  \thanks{Dennis Benders, Robert Babu\v{s}ka, Javier Alonso-Mora and Laura Ferranti are with the department of Cognitive Robotics, Delft University of Technology, 2628 CD Delft, The Netherlands (email: \{d.benders, r.babuska, j.alonsomora, l.ferranti\}@tudelft.nl).}%
  \thanks{Robert Babu\v{s}ka is also with the Czech Institute of Informatics, Robotics and Cybernetics, Czech Technical University in Prague.}%
  \thanks{Johannes K\"{o}hler is with the Institute for Dynamic Systems and Control, ETH Z\"{u}rich, CH-8092, Switzerland (email: jkoehle@ethz.ch).}%
  \thanks{Code available at https://github.com/dbenders1/rohmpc.}%
  }
}

\begin{document}

\maketitle
\thispagestyle{empty}
\pagestyle{empty}

\begin{abstract}
Model predictive control (MPC) is a powerful strategy for planning and control in autonomous mobile robot navigation. However, ensuring safety in real-world deployments remains challenging due to the presence of disturbances and measurement noise. Existing approaches often rely on idealized assumptions, neglect the impact of noisy measurements, and simply heuristically guess unrealistic bounds. In this work, we present an efficient and modular robust MPC design pipeline that systematically addresses these limitations. The pipeline consists of an iterative procedure that leverages closed-loop experimental data to estimate disturbance bounds and synthesize a robust output-feedback MPC scheme. We provide the pipeline in the form of deterministic and reproducible code to synthesize the robust output-feedback MPC from data. We empirically demonstrate robust constraint satisfaction and recursive feasibility in quadrotor simulations using Gazebo.
\end{abstract}

\section{Introduction}\label{sec:intro}
Mobile robots have significant potential to meet societal needs \cite{siegwart2011introduction}, for example, in autonomous search and rescue operations \cite{delmerico2019current}, intralogistics \cite{fragapane2021planning}, and self-driving vehicles \cite{paden2016survey}. These applications require robots to navigate autonomously while avoiding collisions. To address this challenge, various planning techniques have been developed, including reactive \cite{khatib1986real}, sampling-based \cite{karaman2011sampling}, and optimization-based methods \cite{tordesillas2021faster}. These planners typically use nonlinear dynamical system models, which, despite their accuracy, cannot capture all real-world behaviors, leading to model uncertainty. 
In addition, real-world experiments cannot provide accurate state information and only noisy estimates can be derived from sensor data.
If not explicitly considered, such errors can jeopardize robot safety, potentially causing crashes \cite{hwang2024safe}. Therefore, this work aims to ensure collision-free navigation despite uncertainties in the nonlinear robot model.

Several methods have been proposed to handle model uncertainty and ensure safe motion planning. These include control barrier functions \cite{ames2016control,rosolia2022unified}, reachability-based methods \cite{kousik2020bridging,chen2021fastrack}, and model predictive control (MPC) \cite{rawlings2017model}. MPC approaches can be categorized into stochastic MPC, providing probabilistic safety guarantees, and robust MPC, offering deterministic safety guarantees \cite{kouvaritakis2015model}.

This work focuses on robust MPC, optimizing trajectories to remain feasible and collision-free under worst-case disturbances. 
To mitigate conservatism, feedback controllers are used \cite{bemporad1998reducing}, for example in funnel-based approaches \cite{tedrake2010lqr,majumdar2017funnel} and tube MPC \cite{mayne2011tube,lopez2019dynamic,rakovic2022homothetic}. For nonlinear robot dynamics, recent approaches leverage contraction metrics \cite{tsukamoto2021contraction} to derive bounds on the tracking error that are then leveraged in the MPC, see~\cite{singh2023robust,zhao2022tube,sasfi2023robust}.

\begin{figure}[t]
    \centering
    \includesvg[inkscapelatex=false,width=\columnwidth]{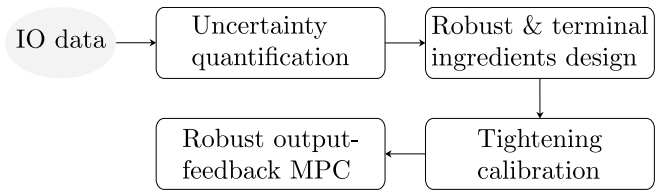}
    \caption{Main contribution: a robust output-feedback MPC design pipeline, starting from input-output (IO) data up to the synthesis of a robust output-feedback MPC scheme.}
    \label{fig:pipeline}
\end{figure}

Designing robust MPC schemes requires knowledge of the worst-case model uncertainty, which includes disturbances, parametric uncertainty, and measurement noise. State-of-the-art methods often assume simplified uncertainty sets, such as linear drag models \cite{singh2023robust}, norm-bounded wind disturbances \cite{zhao2022tube}, and polytopic wind disturbances and mass uncertainty \cite{sasfi2023robust}. These assumptions can be satisfied in simulations, but are not realistic in experimental environments. 
Although, e.g.,~\cite{singh2023robust}, has empirically demonstrated safety in hardware experiments using a quadrotor, this also relied on heuristically added safety margins of over half a meter. %reference: page 25, 52cm 
To facilitate the reliable deployment of robust MPC in real-world applications, this work aims to infer bounds on the model uncertainty directly from noisy measured data.

Obtaining accurate and realistic bounds on uncertainties is crucial for the performance and reliability of robust MPC schemes. Methods to quantify model uncertainty include learning-based approaches such as Gaussian Processes (GPs) and neural networks~\cite{lee2022trust}, which use statistical tools to provide probabilistic bounds. However, these tools rely on strong assumptions about the distribution of the noise and can typically not deal with noisy inputs (cf.~\cite{mchutchon2011gaussian}). 
Distribution-agnostic conformal prediction \cite{lindemann2024formal} offers another approach, though it requires samples of the uncertainty. 
Set-membership estimation \cite{lauricella2020set} also provides distribution-agnostic bounds, but is susceptible to outliers.  
In general, the key challenge in quantifying uncertainty relates to the fact that the model error and states are not measured. 
Popular approaches for estimation include the extended Kalman filter and particle filter~\cite{barfoot2024state}, but they suffer from linearized approximations or sample inefficiency, respectively. Moving horizon estimation (MHE) \cite[Chap.~4]{rawlings2017model} is a modern optimization-based estimation method, which can be applied to general nonlinear robot dynamics to compute optimal estimates. 

In summary, despite the strong theoretical foundation of robust MPC, applications typically rely on simplifying assumptions about the model uncertainty that do not accurately reflect experiments.  To overcome these limitations and move towards efficient and reliable applications of robust MPC, we contribute an efficient and modular design pipeline, which is visualized in Figure~\ref{fig:pipeline}. Our main contributions are:
\begin{enumerate}
    \item Uncertainty quantification: A scheme to determine bounds on the model uncertainty for general nonlinear systems using a simple iteration involving an MHE formulation.
    \item A robust output-feedback MPC-based planning and tracking framework (ROHMPC) that ensures safety and recursive feasibility in the presence of disturbances and measurement noise (Theorem~\ref{thm:rohmpc}). 
    \item An efficient, modular, and reproducible software pipeline, identifying model uncertainty,  computing robust and terminal ingredients, yielding a robust MPC implementation that can be directly deployed. 
\end{enumerate}
We apply the pipeline in simulation experiments with the Gazebo \emph{RotorS} package \cite{furrer2016rotors} including significant structural model mismatch and noisy measurements, yielding an empirically safe robust MPC within 2 hours.

\section{Problem formulation}\label{sec:problem}
In this work, we consider a mobile robot described by the following Lipschitz continuous nonlinear system:
\begin{subequations}\label{eq:sys}
    \begin{align}
        \dx_t &= \fw(\x_t,\bu_t,\w_t) \coloneqq f(\x_t, \bm{u}_t) + E \w_t,\label{eq:sys_xdot}\\
        \y_t &= C \x_t + F \bmeta_t,\label{eq:sys_y}
    \end{align}
\end{subequations}
 with time $t \in \R$, state $\x_t \in \Rx$, input $\bm{u}_t \in \Ru$, output $\y_t \in \Ry$, disturbance selection matrix $E \in \RE$, disturbance $\w_t \in \Rw$, output matrix $C \in \Rytx$, measurement noise selection matrix $F \in \RF$, and measurement noise $\bmeta_t \in \Reta$. The system should satisfy polytopic state and input constraints
\begin{align}\label{eq:con_sys}
    \Z &\coloneqq \X \times \U \notag\\
    &= \left\{(\x,\bm{u}) \in \Rxu | \gjs(\x,\bm{u}) \leq 0, \ j \in \Ns\right\},
\end{align}
with $\gjs(\x,\bm{u}) = \Ljs \begin{bmatrix}\x\\\bm{u}\end{bmatrix} - \ljs, \Ljs \in \mathbb{R}^{1 \times (n^\mathrm{x}+n^\mathrm{u})}, \ljs \in \R, j \in \Ns$.

The goal is to design an MPC scheme that navigates the mobile robot from its current position $\p_t = M \x_t \in \Rp$ to goal position $\pg \in \Rp$ without colliding with obstacles, where $\Rp=\mathbb{R}^3$. To avoid obstacles, we leverage time-varying polytopic sets representing the free spaces between obstacles, given by \cite{benders2025embedded}:
\begin{equation}\label{eq:con_obs}
    \F_{\tt} \coloneqq \left\{\p_{\tt}\ \middle|\ \gjtto(\p_{\tt}) \leq 0,\ j \in \No\right\},
\end{equation}
with $\gjtto(\p_{\tt}) = \Ljtto \p_{\tt} - \ljtto, \Ljtto \in \mathbb{R}^{1 \times n^\mathrm{p}}$, $\ljtto \in \mathbb{R},$ $j \in \No$ and where $\p_{\tt}$ indicates the position at MPC prediction time $\tau$ at run time $t$. Without loss of generality, we set $\norm{\Ljtto} = 1, j \in \No$ by scaling the constraints.

\textit{General approach:} We use a robust output-feedback MPC formulation to ensure collision avoidance despite the presence of disturbances and measurement noise. This formulation tightens the constraints based on the impact of the worst-case disturbance on the system. Correspondingly, we would like to obtain a bound in the form $\w \in \W$ and $\bmeta \in \H$, with polytopic disturbance and noise sets $\W$, $\H$. However, the sets $\W$ and $\H$ are not trivial to obtain since both $\w$ and $\bmeta$ cannot be directly measured on the real system. Therefore, we propose a method to offline estimate $\W$ and $\H$ using an iterative MHE approach, as described in Section~\ref{sec:uq}. These bounds are then used in the robust output-feedback MPC design described in Section~\ref{sec:rohmpc}. This section also provides the theoretical analysis of the scheme, ensuring recursive feasibility, robust constraint satisfaction, and a bounded tracking error. Finally, Section~\ref{sec:results} demonstrates the results of the proposed uncertainty quantification method and the closed-loop properties of the ROHMPC scheme using a quadrotor simulation in Gazebo.

\section{Uncertainty quantification for general nonlinear systems}\label{sec:uq}
Given model \eqref{eq:sys} and a time series of IO data $(\bu,\y)$ of length $\Tmeas$, the goal of this section is to quantify the model uncertainty, i.e., estimate disturbance $\w$ and measurement noise $\bmeta$. Specifically, we find estimated values $\hw$ and $\hbmeta$ using the following formulation:

\begin{subequations}\label{eq:mhe}
    \begin{alignat}{2}
        \underset{\hx_{\ct},\hw_{\ct},\hbmeta_{\ct}}{\operatorname{min}}\ \ &\mathrlap{\int_{\tau=0}^{\Tmhe} \norm{\hw_{\tt}}_Q^2 + \norm{\hbmeta_{\tt}}_R^2\ d\tau,}&&\hspace{1pt}\label{eq:mhe_obj}\\
        \operatorname{s.t.}\ &\dhx_\tt &&= f(\hx_{\tt},\bm{u}_{\tpt}) + E\hw_{\tt},\label{eq:mhe_x}\\
        &\y_{\tpt} &&= C\hx_{\tt} + F\hbmeta_{\tt},\label{eq:mhe_y}\\
        &\tau &&\in [0,\Tmhe],\notag
    \end{alignat}
\end{subequations}
which is a receding horizon estimation scheme with prediction horizon $\Tmhe$ and weighting matrices $Q$ and $R$, comparable to MHE \cite[Chap.~4]{rawlings2017model}. To obtain estimates $\hw$ and $\hbmeta$ that are representative of the real $\w$ and $\bmeta$ encountered during closed-loop experiments, it is beneficial to cover a large part of the state-input space $\Z$. This requires a larger data length $\Tmeas$. However, collecting more data increases the computational complexity. To prevent intractability of \eqref{eq:mhe}, we use a receding horizon implementation with horizon length $\Tmhe < \Tmeas$. Weighting matrices $Q$ and $R$ should reflect the inverse covariance matrices of the disturbance and measurement noise. Since their exact values are unknown, we repeatedly solve \eqref{eq:mhe} and update $Q$ and $R$ until convergence of their eigenvalues. This iteration is comparable to the expectation-maximization (EM) algorithm classically used for identification~\cite{gibson2005robust}. In particular, the optimized cost in~\eqref{eq:mhe} is proportional to the neg-log likelihood (assuming noise is Gaussian distributed), and as we see in the experiments (Sec.~\ref{sec:results}), this iteration monotonically improves estimates. 
After this loop, we compute the bounding box $\hW$ containing all estimated disturbance samples $\hw_{\tt}, \tau=\frac{\Tmhe}{2}, t \in [0,\Tmeas]$, since the estimated values in the middle of the horizon are the most accurate, and set the model bias $\hwb$ equal to its center. The bias can be used to increase the accuracy of the nominal model. Bounding box $\hH$ is computed similarly, however, centered around $\bm{0}$. The following algorithm summarizes the procedure:

\begin{algorithm}
    \caption{Uncertainty quantification}\label{alg:mhe}
    \begin{algorithmic}[1]
        \StateNoIndent \textbf{Input:} $\bm{u}$, $\y$, $f$, $E$, $C$, $F$, $\Tmhe$, $Q > 0$, $R > 0$
        \While{$\mathrm{eig}(Q)$ and $\mathrm{eig}(R)$ not converged}
            \For{$t\in\{\Tmhe,\ldots,\Tmeas\}$}
                \State Solve \eqref{eq:mhe} with $Q$ and $R$
                \State Store $\hw$ and $\hv$
            \EndFor
            \State Update $Q = \operatorname{cov}(\hw)$
            \State Update $R = \operatorname{cov}(\hv)$
        \EndWhile
        \State Compute bounding box $\hW$ and bias $\hwb$: $\hw\in\hwb\oplus\hW$
        \State Compute bounding box $\hH:$ $\hv\in \hH$
    \end{algorithmic}
\end{algorithm}

\section{Robust output-feedback MPC design}\label{sec:rohmpc}
The goal of this section is to describe the robust MPC design that ensures safe and autonomous navigation towards $\pg$ while robustly avoiding collisions in the presence of $\w$ and $\bmeta$. To this end, we summarize the robust output-feedback MPC theory described in \cite{step2025guide} and connect it to the uncertainty quantification method from previous section. Furthermore, we propose the ROHMPC scheme, which is based on the HMPC framework in \cite{benders2025embedded}, using a co-designed planning MPC (PMPC) and tracking MPC (TMPC) scheme.
The PMPC formulation only requires an adjusted model and tightening margins. Therefore, the focus in this section is on the robust TMPC design. Section~\ref{subsec:tmpc_robust} describes the offline robust output-feedback design. This includes the combined design of an observer and an incremental Lyapunov function with the corresponding feedback law. The upper bound of this Lyapunov function, called the tube size, is used to formulate a simple constraint-tightening TMPC scheme in Section~\ref{subsec:tmpc_formulation}. Section~\ref{subsec:tmpc_term} describes the offline design of suitable terminal ingredients used to prove robust constraint satisfaction and recursive feasibility of the ROHMPC scheme in Section~\ref{subsec:analysis}.

\subsection{Offline robust output-feedback design}\label{subsec:tmpc_robust}
The main challenge in designing robust output-feedback MPC schemes is to guarantee that the closed-loop system satisfies the original constraints in the presence of disturbances and measurement noise. To solve this problem, we compute an estimated state $\hx$ based on noisy measurements $\y$ and make sure that both controller error $\bdelta\coloneqq\hx-\z$, with respect to nominal trajectory $\z$, and observer error $\beps\coloneqq\x-\hx$, with respect to actual state $\x$, are bounded.

To ensure boundedness of $\bdelta$, we design the control law
\begin{equation}\label{eq:u_cl}
    \bu_{\tpt} \coloneqq \bv_{\tt}^* + \kappad(\hx_{\tpt},\z_{\tt}^*),\ \tau \in [0,\Ts],
\end{equation}
which consists of a nominal input $\bv_{\tt}^*$, recomputed every $\Ts$ seconds by solving the TMPC formulation described in Section~\ref{subsec:tmpc_formulation}, and a continuous feedback law $\kappad(\hx_{\tpt},\z_{\tt}^*)$ that is applied between these discrete sampling times. Furthermore, to compute $\hx$ and ensure the boundedness of $\beps$, we design an observer with the following dynamics:
\begin{equation}\label{eq:observer}
    \dhx_t \coloneqq f(\hx_t,\bu_t) + E\hwb + L(\y_t-\hx_t),
\end{equation}
with observer gain $L\in\RL$.

\begin{figure*}[!ht]
    \normalsize
    \begin{subequations}\label{eq:sdp}
        \begin{align}
            \underset{\substack{\Xd,\Ydarg,\\{\cjs}^2,{\co}^2,\epsilon^2}}{\opmin}\ \ &\cco {\co}^2 + \sum_{j=1}^{\ns} \cjcs {\cjs}^2 + c^\epsilon \epsilon^2,\label{eq:sdp_obj}\\
            \operatorname{s.t.}\ &\Xd \succeq 0,\label{eq:sdp_lmi_x}\\
            &\Aarg\Xd+\Barg\Ydarg+\left(\Aarg\Xd+\Barg\Ydarg\right)^\top+2\rho \Xd \preceq 0,\label{eq:sdp_lmi_contr}\\
            &\begin{bmatrix}
                \Aarg\Xd+\Barg\Ydarg+\left(\Aarg\Xd+\Barg\Ydarg\right)^\top+\lambdadelta\Xd&LC\Xd&LF\bmeta\\
                \left(LC\Xd\right)^\top&-\lambdadeltaeps\Xd&\bm{0}\\
                \left(LF\bmeta\right)^\top&\bm{0}&\lambdadeltaeps\epsilon^2-\lambdadelta\delta^2
            \end{bmatrix} \preceq 0,\label{eq:sdp_lmi_rpi_delta}\\
            &\begin{bmatrix}
                \Xd\left(\Aarg-LC\right)^\top + \left(\Aarg-LC\right)\Xd + \lambdaeps\Xd&E\w^0-LF\bmeta\\
                \left(E\w^0-LF\bmeta\right)^\top&-\lambdaeps\epsilon^2
            \end{bmatrix} \preceq 0,\label{eq:sdp_lmi_rpi_eps}\\
            &\begin{bmatrix}
                {\cjs}^2&\Ljs\begin{bmatrix}\Ydarg\\\Xd\end{bmatrix}\\
                \left(\Ljs\begin{bmatrix}\Ydarg\\\Xd\end{bmatrix}\right)^\top&\Xd
            \end{bmatrix} \succeq 0,\ j \in \Ns,\label{eq:sdp_lmi_sys}\\
            &\begin{bmatrix}
                {\co}^2 I^{\np}&M\Xd\\
                \left(M\Xd\right)^\top&\Xd
            \end{bmatrix} \succeq 0,\label{eq:sdp_lmi_obs}\\
            &\forall \bzeta\coloneqq\begin{bmatrix}
                \bu\\\x
            \end{bmatrix}\in\Z, \quad \w^0\in\opvert(\hW), \quad \bmeta\in\opvert(\hH).\notag
        \end{align}
    \end{subequations}
\end{figure*}

Leveraging robust MPC designs using contraction metrics \cite{singh2023robust,zhao2022tube,sasfi2023robust}, including their extension to the output-feedback case \cite{kohler2019simple,step2025guide}, we jointly optimize for an incremental Lyapunov function with constant metric $\Pd\in\mathbb{R}^{\mathrm{n^x}\times \mathrm{n^x}}$ and corresponding feedback law $\kappadhxz$ by solving the semi-definite program (SDP) \eqref{eq:sdp}; see \cite{step2025guide} for more details. This problem uses a given observer gain $L$, and treats the contraction rate $\rho>0$ and multipliers $\lambdadelta,\lambdadeltaeps,\lambdaeps\geq 0$ as hyperparameters. Furthermore, the Jacobians in \eqref{eq:sdp} are given by
\begin{equation}\label{eq:A_B}
    A(\bzeta) = \left.\frac{\partial \fw(\x,\bu,\w)}{\partial \x}\right|_{\bzeta}, \quad B(\bzeta) = \left.\frac{\partial \fw(\x,\bu,\w)}{\partial \bu}\right|_{\bzeta},
\end{equation}
The resulting incremental Lyapunov function,
\begin{equation}
    \Vdhxz \coloneqq (\hx-\z)^\top \Pd (\hx-\z),
\end{equation}
quantifies controller error $\bdelta$ and is guaranteed to contract by at least rate $\rho$ under robust feedback law
\begin{equation}\label{eq:kappad}
    \kappad(\hx,\z) = \bv + \int_{0}^{1} \Kdgs ds \bdelta,
\end{equation}
where $\gxs = \z + s\bdelta$ is the geodesic, which is a straight line connecting $\hx$ and $\z$ with derivative $\gxss = \bdelta$.
The matrices $\Pd$ and $\Kd(\gxs)$ are obtained by the following change of coordinates:
\begin{equation}
    \Pd = \Xd^{-1},\quad \Kd(\gxs) = \Yd(\gxs) \Pd,
\end{equation}
This contractivity property is imposed by the linear matrix inequality (LMI) \eqref{eq:sdp_lmi_contr}.

\begin{remark}\label{rem:state-dependency}
    Note that one could use a state-dependent metric \cite{singh2023robust,zhao2022tube,sasfi2023robust} to reduce conservatism. In this case, evaluating the geodesic $\gxs$ and the Lyapunov function $\Vdhxz$ would require solving an optimization problem. Since we 
    Since we require the evaluation of $\Vdhxz$ in the terminal set constraint \eqref{eq:tmpc_term}, this would have resulted in a significant increase in the computational demand. 
    In the special case of a constant feedback matrix $\Kd$, the control law \eqref{eq:u_cl} reduces to $\bv + \Kd\bdelta$, which is also leveraged in the experiments later.
\end{remark}
We bound $\Vdhxz$ from above by tube size $s$ as
\begin{equation}\label{eq:s_tube}
    \sqrt{\Vd(\hx_\tt,\z_\tt)} \leq s_\tau,\ t\geq 0,\ \tau\in[0,T],
\end{equation}
with $\dot{s}_\tau=-\rho s_\tau + \bwo$. Thus, controller error $\bdelta$ is bounded by tube size $s$. Furthermore, $s$ grows with bounded factor $\bwo > 0$ and contracts with rate $\rho$. Due to its $\rho$-contractivity, $s$ cannot grow unbounded, making it suitable as a constraint tightening factor for system constraints \eqref{eq:con_sys} and obstacle avoidance constraints \eqref{eq:con_obs}.

To show constraint satisfaction for $\x$, we also need to account for the difference between $\hx$ and $\x$. By the UQ method in Section~\ref{sec:uq}, we know that the measurement noise is bounded, i.e., $\hbmeta\in\hH$. Combined with observer \eqref{eq:observer}, we can construct a constant tube $\epsilon$ bounding $\beps$ as
\begin{equation}\label{eq:eps_tube}
    \sqrt{\Vd(\x_t,\hx_t)} \leq \epsilon,\ t\geq 0.
\end{equation}
Thus, we ensure satisfaction of constraints in \eqref{eq:con_sys} and \eqref{eq:con_obs} by tightening the constraints on the nominal trajectory $z$ proportional to the two error bounds $\epsilon,s$ using suitable constraint tightening constants. 
The SDP \eqref{eq:sdp} optimizes for the smallest constraint tightening constants $\cjs,\cjso,j\in\Ns$, with
\begin{equation}
    \cjso = \begin{cases}
        0, &j\in\N_{[1,2\nuu]}\\
        \cjs, &j\in\N_{[2\nuu+1,\ns]}
    \end{cases},
\end{equation}
and $\co$. This is imposed by the combination of \eqref{eq:sdp_obj}, \eqref{eq:sdp_lmi_rpi_delta}, \eqref{eq:sdp_lmi_rpi_eps}, \eqref{eq:sdp_lmi_sys}, and \eqref{eq:sdp_lmi_obs}. To avoid conservatism in certain directions of the state space, the constraint tightening constants are normalized by their corresponding constant constraint intervals $\cjcs$ and $\cco$ in \eqref{eq:sdp_obj}. In this case, $\cco$ represents the minimum desired distance to obstacles. Moreover, $c^\epsilon$ is a penalty term to reduce $\epsilon$.

Proposition~\ref{prop:robust} formalizes the properties of the robust output-feedback design:
\begin{proposition}[Robust output-feedback design]\label{prop:robust}
Consider a nominal trajectory $(\z,\bv)$ satisfying \eqref{eq:con_sys} tightened by $\cjs s_\tau + \cjso \epsilon, j\in\Ns$ and \eqref{eq:con_obs} tightened by $\co (s_\tau+\epsilon), j\in\No$. Then, the closed-loop trajectory $(\x,\bu)$ under robust feedback law \eqref{eq:u_cl} and observer \eqref{eq:observer} subject to disturbances $\hw\in\hwb\oplus\hW$ and measurement noise $\hbmeta\in\hH$ satisfies system constraints \eqref{eq:con_sys} and obstacle avoidance constraints \eqref{eq:con_obs}.
\end{proposition}
\begin{proof}
    See \cite{step2025guide}.
\end{proof}

\subsection{TMPC formulation}\label{subsec:tmpc_formulation}
Given the robust output-feedback quantities designed in the previous section, we formulate the TMPC problem at time $t$ as follows:
\begin{subequations}\label{eq:tmpc}
    \begin{alignat}{2}
        \underset{\substack{\z_{\cdot|t},\bv_{\cdot|t}}}{\operatorname{min}}\ \ &\mathrlap{\Jf(\z_{\Tt},\x_{\tpT}^\mathrm{r}) + \int_{0}^{T} \Js(\z_{\tt},\bv_{\tt},\r_{\tpt})\ d\tau,}&&\hspace{200pt} \label{eq:tmpc_obj}\\
        \operatorname{s.t.}\ &\z_{\0t} = \hx_t,&& \label{eq:tmpc_z0}\\
        &s_0 = 0,&& \label{eq:tmpc_s0}\\
        &\dz_{\tt} = f(\z_{\tt},\bv_{\tt}) + E\hwb,&& \label{eq:tmpc_zdot}\\
        &\ds_\tau = -\rho s_\tau + \bwo,&& \label{eq:tmpc_sdot}\\
        &\gjs(\z_{\tt},\bv_{\tt}) + \cjs s_\tau + \cjso \epsilon \leq 0,&&\ j \in \Ns \label{eq:tmpc_sys}\\
        &\gjtto(M\z_{\tt}) + \co (s_\tau+\epsilon) \leq 0,&&\ j \in \No \label{eq:tmpc_obs}\\
        &(\z_{\Tt},s_T,\epsilon) \in \Xf(\xr),&&\label{eq:tmpc_term}\\
        &\tau \in [0,T],\notag
    \end{alignat}
\end{subequations}
with reference trajectory $\r=[{\xr}^\top {\ur}^\top]^\top$, stage cost $\Js(\z,\bv,\r)=\norm{\z-\xr}_Q^2+\norm{\bv-\ur}_R^2, Q \succ 0, R \succ 0$, terminal cost $\Jf(\z,\xr)=\norm{\z-\xr}_P^2, P \succ 0$, and prediction horizon $Tt\geq \Ts0$. In this formulation, $Q$ and $R$ can be manually tuned. Terminal cost $\Jf(\z,\xr)$ and terminal set $\Xf(\xr)$ are designed as described in Section~\ref{subsec:tmpc_term}.

\subsection{Offline terminal ingredients design}\label{subsec:tmpc_term}
The goal of this section is to find a terminal control law that renders $\Xf(\xr)$ invariant and ensures that we know a lower bound on the decrease of $\Jf(\z,\xr)$ over time. These properties are used to prove the recursive feasibility and trajectory tracking in Section~\ref{subsec:analysis}, respectively.

The following proposition states a suitable design for terminal set $\Xf(\xr)$ and terminal cost $\Jf(\z,\xr)$:
\begin{proposition}[Terminal ingredients]\label{prop:term}
    Terminal control law $\kappafhxr\coloneqq\ur+\kappad(\hx,\xr)$ ensures that:
    \begin{enumerate}
        \item there exists a terminal set scaling $\alpha>0$ such that
        \begin{align}\label{eq:Xf}
            &\Xf(\xr) \coloneqq\notag\\
            &\left\{\z \in \X, s \in \R, \epsilon \in \R \ \middle|\ \sqrt{\Vd(\z,\xr)} + s_T + \epsilon \leq \alpha \right\},
        \end{align}
        is robustly positive invariant under control law $\kappafhxr$ with $\xr$ satisfying \eqref{eq:tmpc_sys} and \eqref{eq:tmpc_obs} with $s_\tau$ replaced by the invariant tube size $s_\infty=\frac{\bwo}{\rho}$,
        \item tightened system constraints \eqref{eq:tmpc_sys} and obstacle avoidance constraints \eqref{eq:tmpc_obs} are satisfied in $\Xf(\xr)$;
        \item terminal cost $\Jf(\z,\xr)$ satisfies the following decrease property:
        \begin{equation}\label{eq:Jf_decrease}
            \begin{aligned}
                &\Jf(\z_{\TTt}^*,\x_{\tpTsT}\tr) - \Jf(\z_{\Tt}^*,\x_{\tpT}\tr) \leq\\
                &-\int_{T}^{T+\Ts}\Js(\z_{\tt}^*,\bv_{\tt}^*,\r_{\tpt}) d\tau,
            \end{aligned}
        \end{equation}
        with terminal cost matrix $P$ computed using SDP \eqref{eq:sdp_P}.
    \end{enumerate}
\end{proposition}
\begin{proof}
    1) is proven by showing that terminal set constraint \eqref{eq:tmpc_term} is satisfied at $T+\tau,\tau\geq 0$ given that it is satisfied at $\tau=0$. We know that $\sqrt{\Vd(\z,\xr)}$ contracts with factor $\rho$ under feedback law \eqref{eq:kappad}. It can be shown that this reduction cancels out against the increase in tube size $s_{\Tpt}$ for $\tau>0$. Since $\epsilon$ is constant, \eqref{eq:tmpc_term} is satisfied for $\tau\geq 0$. This result holds for $\alpha\geq\frac{\bwo}{\rho}+\epsilon$. See \cite{step2025guide} for more details.
    2) Since $s_T$ and $\epsilon$ are contained in $\Xf(\xr)$, tightened system constraints \eqref{eq:tmpc_sys} and obstacle avoidance constraints \eqref{eq:tmpc_obs} are satisfied in $\Xf(\xr)$.
    3) LMI \eqref{eq:sdp_P_lmi_decrease} in SDP \eqref{eq:sdp_P} enforces that $\Jf(\z,\xr)$ also decreases according to \eqref{eq:Jf_decrease} \cite{benders2025embedded}.
\end{proof}

\begin{figure*}[!ht]
    \normalsize
    \begin{subequations}\label{eq:sdp_P}
        \begin{align}
            \underset{P}{\opmin}\ &\optrace\,P,\label{eq:sdp_P_obj}\\
            \operatorname{s.t.}\ &P \succeq 0,\label{eq:sdp_P_lmi_P}\\
            &\left(A(\bzeta)+B(\bzeta)\Kd\right)^\top P + P\left(A(\bzeta)+B(\bzeta)\Kd\right) + Q + \Kd^\top R\Kd \preceq 0,\label{eq:sdp_P_lmi_decrease}\\
            &\forall \bzeta \in \Z.\notag
        \end{align}
    \end{subequations}
\end{figure*}

\subsection{Theoretical analysis}\label{subsec:analysis}
Theorem~\ref{thm:rohmpc} formalizes the theoretical guarantees associated with the proposed ROHMPC framework:

\begin{theorem}\label{thm:rohmpc}
    Suppose the robust and TMPC terminal ingredients are designed offline according to Sections~\ref{subsec:tmpc_robust} and~\ref{subsec:tmpc_term}, the system starts at steady-state, i.e., $\dx_0 = 0$, and both the PMPC problem \cite{benders2025embedded} and TMPC problem \eqref{eq:tmpc} are feasible at $t=0$. Then, both problems are recursively feasible, the closed-loop system \eqref{eq:u_cl} satisfies system constraints $(\x_t,\bu_t) \in \Z$ and avoids obstacle collisions $M\x_t \notin \O$ for all $t \geq 0$. Moreover, the average tracking error $\norm{\x_t-\x_t\tr}$ is bounded over time.
\end{theorem}
\begin{proof}
    The PMPC is recursively feasible following the proof provided in \cite{benders2025embedded}. Tightening the system and obstacle avoidance constraints in the PMPC formulation with positive invariant tube $\alpha$ guarantees that the reference trajectory satisfies the robust equivalent of the properties detailed in \cite{benders2025embedded}. Consequently, the combination of Proposition~\ref{prop:term} and the recursive feasibility proof in \cite{step2025guide} guarantees that the TMPC, and thus the ROHMPC framework, is recursively feasible. Therefore, the closed-loop system \eqref{eq:u_cl} robustly satisfies the system constraints and obstacle avoidance constraints. Furthermore, the average tracking error is bounded over time as shown by a combination of Proposition~\ref{prop:term} and the trajectory tracking proof in \cite{step2025guide}.
\end{proof}

\section{Results}\label{sec:results}
This section first presents relevant implementation details (Section~\ref{subsec:results_implementation}) and continues with the results of the proposed uncertainty quantification method (Section~\ref{subsec:results_uq}), the robust output-feedback design and tightening calibration (Section~\ref{subsec:results_robust_design}), and ROHMPC scheme (Section~\ref{subsec:results_rohmpc}).

\subsection{Implementation details}\label{subsec:results_implementation}
This section provides the nominal quadrotor model and its corresponding constraints in Section~\ref{subsubsec:results_model}, the model uncertainty description in Section~\ref{subsubsec:results_uncertainty}, and the software setup used to generate the results in the next sections in Section~\ref{subsubsec:results_software}.

\subsubsection{Quadrotor model and constraints}\label{subsubsec:results_model}
The nominal quadrotor model is based on~\cite{sun2022comparative}:
\begin{align}\label{eq:quad_model}
    \dotp &= \bv\notag\\
    \begin{bmatrix}\dot{\phi}&\dot{\theta}&\dot{\psi}\end{bmatrix}^\top &= {\BIRc}{\bOmega}\notag\\
    \dv &= \BIRr \begin{bmatrix}0&0&\frac{T}{m}\end{bmatrix}^\top - \begin{bmatrix}0&0&g\end{bmatrix}^\top\notag\\
    I\dOmega &= -\bOmega \times I\bOmega + \btau,
\end{align}
with states including 3D positions in inertial frame $\I$, ZYX Euler angles, velocities and angular velocities in the body frame, thrust and torques defined as a linear function of the inputs $\bu = [t_0, t_1, t_2, t_3]^\top$ in which $t_i$ is the thrust of rotor $i$, gravitational constant $g=9.8124$, mass $m=0.617$, inertia matrix $I=\opdiag(0.00164,0.00184,0.0030)$, and coordinate and rate rotation matrices $\BIRc$ and $\BIRr$ to convert coordinates and rates from body to inertial frame. Sine, cosine, and tangent functions are defined as $s\theta = \sin(\theta)$, $c\theta = \cos(\theta)$, $t\theta = \tan(\theta)$, for example.

The quadrotor is subject to the following constraints, derived from the trajectories used to generate the results in Section~\ref{subsec:results_uq}:
\begin{alignat}{2}\label{eq:quad_constraints}
    -4\ \mathrm{m} &\leq p^x,\ p^y &&\leq 4\ \mathrm{m},\notag\\
    0\ \mathrm{m} &\leq p^z &&\leq 4\ \mathrm{m},\notag\\
    -2\ \mathrm{m/s} &\leq v^x,\ v^y,\ v^z &&\leq 2\ \mathrm{m/s},\notag\\
    -0.1\ \mathrm{rad} &\leq \phi,\ \theta,\ \psi &&\leq 0.1\ \mathrm{rad},\notag\\
    -0.3\ \mathrm{rad/s} &\leq \Omega^x,\ \Omega^y,\ \Omega^z &&\leq 0.3\ \mathrm{rad/s},\notag\\
    1.3936\ \mathrm{N} &\leq t_0,\ t_1,\ t_2,\ t_3 &&\leq 1.6336\ \mathrm{N}.
\end{alignat}

\subsubsection{Model uncertainty}\label{subsubsec:results_uncertainty}
The model mismatch  $\w$ arises due to the difference between the simulation of \eqref{eq:quad_model} and the dynamics simulated in the \emph{RotorS} simulator, which are estimated in Section~\ref{subsec:results_uq}. Since the \emph{RotorS} simulator is a physics-based simulator that structurally differs from the nominal model, we choose $E=\Inx$ and $w\in\Rx$. The outputs $\y$ are given by a noisy measurement of the state $\x$ with $C=\Inx$, $F=\Inx$, and $\bmeta\in\Rx$. In this case, $\bmeta$ is generated as uniform noise with the following bounds:
\begin{alignat}{2}\label{eq:meas_noise_bounds}
    -0.00005\ \mathrm{m} &\leq \eta^{p^x},\ \eta^{p^y},\ \eta^{p^z} &&\leq 0.00005\ \mathrm{m},\notag\\
    -0.000628\ \mathrm{m/s} &\leq v^x,\ v^y,\ v^z &&\leq 0.000628\ \mathrm{m/s},\notag\\
    -0.0005\ \mathrm{rad} &\leq \phi,\ \theta,\ \psi &&\leq 0.0005\ \mathrm{rad},\notag\\
    -0.00076\ \mathrm{rad/s} &\leq \Omega^x,\ \Omega^y,\ \Omega^z &&\leq 0.00076\ \mathrm{rad/s}.
\end{alignat}

\subsubsection{Software setup}\label{subsubsec:results_software}
The code corresponding to this work is available at
\begin{center}
    \ifthenelse{\boolean{anonymize}}{https://anonymous.4open.science/r/rohmpc-anonymous-0790}{https://github.com/dbenders1/rohmpc}.
\end{center}
In constructing the software pipeline, special care was taken to ensure modularity and reproducibility. One aspect to highlight is the fact that the code contains an adjusted version of the open-source stack Agilicious \cite{foehn2022agilicious} in order to ensure deterministic \emph{RotorS} simulation results. Moreover, the results are trivial to reproduce as a Docker container is provided, in line with the recommendation in \cite{cervera2018try}.

The results presented in the next sections are generated using a Dell XPS 15 laptop with a 12-core 2.60GHz Intel i7-10750H CPU. We leverage the ACADOS SQP solver \cite{verschueren2021acados} to solve UQ problem \eqref{eq:mhe}, Mosek \cite{andersen2000mosek} to solve SDP \eqref{eq:sdp}, and the ForcesPro NLP solver \cite{zanelli2020forces} to solve PMPC problem \cite{benders2025embedded} and TMPC problem \eqref{eq:tmpc}.

\subsection{Uncertainty quantification}\label{subsec:results_uq}
Recall that the purpose of solving UQ problem \eqref{eq:mhe} is to find bounding boxes $\hW$ and $\hH$ that are used in the robust output-feedback MPC design in Section~\ref{sec:rohmpc}. We solve \eqref{eq:mhe} using IO data obtained by flying 4 different trajectories using the built-in MPC \cite{sun2022comparative} in Agilicious in the \emph{RotorS} simulator. These trajectories are designed to cover a relevant part of the state-input space and consist of clockwise and counter-clockwise circles and lemniscates with radius 1 m and frequency 0.3 Hz.

Note that there is a fundamental limitation in estimating $\w,\bmeta$ using \eqref{eq:mhe}: $\y$ can be explained by both the effect of $\w$ and $\bmeta$, which introduces ambiguity. As a result, the solution to \eqref{eq:mhe} practically converges to either explaining $\y$ mainly by $\w$ or by $\bmeta$, depending on the initialization of $Q$ and $R$. To mitigate this issue, we start from $Q=\Inx$ and $R=\Ineta$ and assume that $\H$ is known, i.e., it can be derived from the sensor datasheets, and add them as constraints to \eqref{eq:mhe}. In this work, we use the bounds in \eqref{eq:meas_noise_bounds}.

Figure~\ref{fig:uq_mle} demonstrate that solving \eqref{eq:mhe} gives a reasonable result: it maximizes the likelihood, yielding a maximum likelihood estimate (MLE), making the estimates the most likely explanation. Since the results are based on the \emph{RotorS} simulator, we have access to the ground truth states. Therefore, we can run \eqref{eq:mhe} with $\y=\x$ and ignore $\bmeta$ to obtain the ground truth $\w$ and $\W$. This allows to compare the estimated disturbance bounds with the ground truth ones, giving a ratio of one in case of perfect estimation. Figure~\ref{fig:uq_ratios} visualizes this value by the dashed green line. Moreover, Figure~\ref{fig:uq_ratios} shows the average of the ratios of lower and upper bounds for each disturbance component with a dash-dotted line and the total mean ratio including one standard deviation with the blue line. Thus, 2 iterations, provide, on average, the best disturbance bound ratios. Correspondingly, we stop Algorithm~\ref{alg:mhe} after 2 iterations, yielding a total computation time of 3408 s. The resulting root mean squared error (RMSE) of all lower and upper disturbance bounds is 0.00156.

\begin{figure}[t]
    \vspace{5pt}
    \centering
    \begin{subfigure}{\columnwidth}
        \centering
        \includegraphics{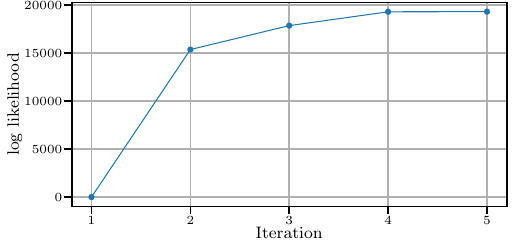}
        \caption{}
        \label{fig:uq_mle}
    \end{subfigure}
    \begin{subfigure}{\columnwidth}
        \centering
        \includegraphics{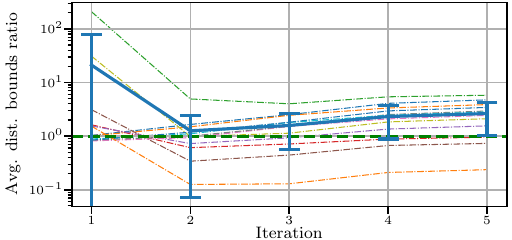}
        \caption{}
        \label{fig:uq_ratios}
    \end{subfigure}
    \caption{Results after 5 iterations of Algorithm~\ref{alg:mhe} based on the clockwise circle trajectory. (a) The log likelihood monotonically increases over the iterations, yielding an MLE. (b) The average ratios of the lower and upper bounds of the estimated disturbance bounds with respect to the ground truth per disturbance component (dash-dotted), the total mean ratio including one standard deviation (blue solid), and the ideal disturbance bound ratio (green dashed). Initially, the ratios are often off by a factor of $100$, but after $2$ iterations most ratios converge relatively close to one and the standard deviation decreases significantly. }
    \label{fig:uq}
\end{figure}

\subsection{Robust output-feedback design}\label{subsec:results_robust_design}
Based on $\hW$ and $\H$ from the previous section, we solve \eqref{eq:sdp} in 295 s. While the resulting feedback $\kappad$ robustly stabilizes the system, the corresponding constraint tightening is overly conservative. To reduce conservatism, we empirically determine the constants used for constraint tightening - $\rho$, $\bwo$, and $\epsilon$ - corresponding to the optimized feedback $\kappad$ and Lyapunov function $\Vdhxz$.

\begin{figure}[t]
    \vspace{5pt}
    \centering
    \includegraphics[width=\columnwidth]{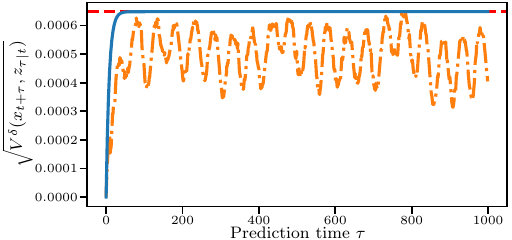}
    \caption{The worst-case Lyapunov error over prediction time $\tau$ (orange dash-dotted), the maximum tube size (red dashed), and the fitted tube based on the calibrated $\rho$ (blue solid).}
    \label{fig:tube_fit}
\end{figure}

To calibrate $\rho$, we track a nominal stable reference trajectory with feedback controller \eqref{eq:kappad} and observer \eqref{eq:observer} and record the nominal reference and the estimated states over time. Using this data, we can compute the tightest s-tube \eqref{eq:s_tube} over multiple steps, as visualized in Figure~\ref{fig:tube_fit}. We repeat this process for multiple grid points of $\rho$ and pick $\rho$ with the smallest tightening. This computation takes 3.5 s using a nominal reference duration of 10 s sampled with the simulator, estimation, and feedback frequency of 500 Hz, and gives $\rho=12.5$.

Given a calibrated $\rho$, we can calibrate $\bwo$ and $\epsilon$ by comparing the estimated state with the one-step ahead predictions of a closed-loop nominal MPC scheme and the ground truth state, respectively. In this simulation, the ground truth state is available. However, in reality one has to pick the more conservative value for epsilon given by SDP \eqref{eq:sdp}. This calibration takes 30.5 s in total, for 144 s of flight time sampled at the TMPC sampling frequency of 100 Hz for $\bwo$ and sampled at 500 Hz for $\epsilon$. The resulting values are $\bwo=0.0337$ and $\epsilon=0.00345$. The trajectories used to calibrate $\bwo$ and $\epsilon$ are more aggressive than the ones produced by the ROHMPC framework described next, so they are reasonable to use to guarantee safety.

\subsection{Closed-loop ROHMPC simulation}\label{subsec:results_rohmpc}
Given the calibrated tightening constants $\rho$, $\bwo$, and $\epsilon$, we can compute $\alpha$ used for the tightening in the PMPC and run the closed-loop ROHMPC scheme. The goal of this scheme is to fly the quadrotor from a starting position to a goal position and safely avoid the obstacle in between. Figure~\ref{fig:rohmpc} visualizes the relevant part of the traversed trajectory (close to the obstacle) and shows that the closed-loop system stays within the total TMPC tube, accurately follows the plan, and safely avoids the obstacle. Furthermore, the total TMPC tube does not grow further than the tube used to tighten the constraints in the PMPC, and all tubes are contained in the obstacle-free space. Finally, the closed-loop system is empirically shown to guarantee constraint satisfaction and recursive feasibility.

\begin{figure}[t]
    \centering
    \includegraphics[width=\columnwidth]{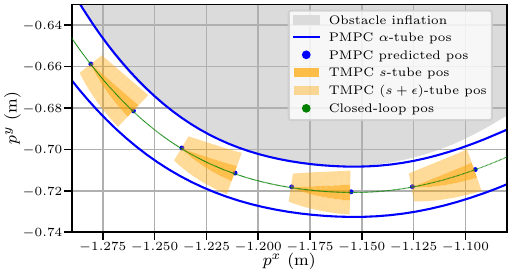}
    \caption{The resulting tubes, plan, and position of the ROHMPC framework in closed-loop simulation with the physics simulator. The predicted tube captures the effect of noise and disturbances, thus robustly ensuring collision avoidance in a systematic fashion.}
    \label{fig:rohmpc}
\end{figure}

\section{Conclusion}\label{sec:conclusion}
Guaranteeing safety on autonomous mobile robots is a challenging topic, especially in the presence of disturbances and measurement noise. Whereas other approaches typically assume specific disturbance effects or known, sometimes unrealistic, bounds on the disturbances and do not consider measurement noise, the aim of this work is to take a step toward relaxing these assumptions and work towards reproducible and robust results on real robotic systems. Therefore, this paper proposed a robust MPC design pipeline to guarantee robust constraint satisfaction and recursive feasibility. Furthermore, it verified these properties using a reproducible software stack in a closed-loop simulation of the proposed ROHMPC framework.

\ifthenelse{\boolean{anonymize}}{}{
  \section*{Acknowledgment}
The authors would like to thank Sihao Sun for the introduction to the Agilicious stack and the fruitful discussions on quadrotor dynamics.

}

\clearpage
\bibliographystyle{IEEEtran}
\bibliography{IEEEabrv,mybib.bib}

% Possibly insert a biography section here

\vfill

\end{document}